\documentclass[letterpaper, 10 pt, conference]{ieeeconf}  % Comment this line out if you need a4paper

\IEEEoverridecommandlockouts                              % This command is only needed if 
\usepackage{graphics} % for pdf, bitmapped graphics files
\usepackage{epsfig} % for postscript graphics files
\usepackage{times} % assumes new font selection scheme installed
\usepackage{amsmath} % assumes amsmath package installed
\usepackage{amssymb}  % assumes amsmath package installed
\usepackage{mathtools}
\usepackage{bm}
\usepackage{nicefrac}
\usepackage{algorithm}
\usepackage{algpseudocode}
\usepackage{makecell}
\usepackage{tikz}

\usetikzlibrary{arrows.meta,positioning,calc,fit}

\usepackage{balance}

\newtheorem{theorem}{Theorem}
\newtheorem{corollary}[theorem]{Corollary}
\newtheorem{proposition}[theorem]{Proposition}
\newtheorem{remark}[theorem]{Remark}
\usepackage[table]{xcolor}
\newcommand{\ourmethod}{ReSQ-MPPI}

\title{\LARGE \bf
Reachability-Guided Sequential Quadratic Programming-Guarded \\ Model Predictive Path Integral for Safe Nonlinear Predictive Control
}

\author{Alexandre Didier$^{1}$, Jason J. Choi$^{2}$, Namhoon Cho$^{3}$, Claire J. Tomlin$^{4}$ and Melanie N. Zeilinger$^{1}$%
\thanks{$^{1}$Institute for Dynamic Systems and Control, ETH Zurich, 8092 Zurich, Switzerland. (email: {\tt\small\{adidier, mzeilinger\}@ethz.ch})}%
\thanks{$^{2}$Electrical and Computer Engineering Department, University of California, Los Angeles,  CA 90095, USA.}%
\thanks{$^{3}$Department of Aerospace Engineering, Seoul National University, Seoul
08826, Republic of Korea.}%
\thanks{$^{4}$Department of Electrical Engineering and Computer Sciences, University of California, Berkeley, CA 94704, USA.}%
}

\begin{document}

\maketitle
\thispagestyle{empty}
\pagestyle{empty}

%%%%%%%%%%%%%%%%%%%%%%%%%%%%%%%%%%%%%%%%%%%%%%%%%%%%%%%%%%%%%%%%%%%%%%%%%%%%%%%%
\begin{abstract}
Safe robot control often requires combining long-horizon performance optimization with hard state and input constraints, but existing approaches tend to address this tradeoff partially. Sampling-based model predictive control (MPC) methods such as model predictive path integral (MPPI) are effective in handling nonlinear and nonconvex environments, yet their finite-sample rollouts and unconstrained weighted-average update can return an unsafe control. Deterministic nonlinear MPC can explicitly incorporate constraints, but its real-time safety and performance depends strongly on warm starts and local convergence. Hamilton--Jacobi reachability (HJR) provides rigorous safety certificates, but offline value-function computation remains practical only for reduced-order models. We propose \ourmethod, a reachability-informed generation--refinement architecture that combines these complementary strengths. An HJR value function computed offline for a reduced-order model guides online MPPI sampling toward safe, promising trajectory candidates. The MPPI solution is then refined by a small number of sequential quadratic programming (SQP) iterations in a full-order MPC problem. A key observation is that the standard MPPI inference step is an unconstrained weighted least-squares problem; \ourmethod\ replaces it with a constrained MPC refinement that recovers the MPPI update when it is feasible and minimally modifies it otherwise. Simulations in cluttered navigation and autonomous racing environments demonstrate that \ourmethod\ improves safety and performance over standalone MPPI, MPC, and reachability-filtered sampling-based control baselines.
\end{abstract}

%%%%%%%%%%%%%%%%%%%%%%%%%%%%%%%%%%%%%%%%%%%%%%%%%%%%%%%%%%%%%%%%%%%%%%%%%%%%%%%%
\section{Introduction}

% Model Predictive Control (MPC) is an effective control framework for complex decisions of various robotics system, such as whole-body control and locomotion, manipulation, and autonomous driving. Its natural handling of safety constraints such as collision avoidance as part of the constraints in the optimization, make it appealing for various tasks associated with complex safety constraints. However, one of the crucial challenges for systems with complex dynamics is the complexity of the nonlinear optimization it inherently entails to solve.
Model predictive control (MPC) has become a powerful framework for solving complex decision-making problems in robotics, including whole-body control and locomotion \cite{di2018dynamic}, mobile manipulation \cite{pankert2020perceptive}, and autonomous driving \cite{williams2018information}. A key advantage of MPC is its ability to naturally incorporate safety requirements, such as collision avoidance and actuator limits, as constraints within the optimization problem, making it well suited for safety-critical robotic applications. However, one of the crucial challenges with MPC for systems with complex dynamics is the complexity of nonlinear optimization that it entails inherently to be solved online.

\textit{Optimization-based methods}, particularly those based on sequential quadratic programming (SQP), have emerged as promising approaches for efficiently solving nonlinear MPC problems \cite{rawlings2017, verschueren2022acados}. However, these methods remain susceptible to infeasibility and rely heavily on a good initialization to achieve rapid convergence. In practice, the optimization is typically warm-started using the solution from the previous control loop. When the solver fails to find a feasible solution within the limited number of optimization iterations in real-time, the resulting control input may violate safety constraints and compromise system safety.

On the other hand, \textit{sampling-based MPC} has emerged as an effective framework for approximately solving nonlinear optimization problems. Model predictive path integral (MPPI) and other variational inference MPC frameworks pose optimal control as a variational inference of the distribution of optimal control trajectories \cite{rawlik2012stochastic}. While these methods can effectively handle complex nonlinear dynamics and constraints within the finite prediction horizon, they lack principled mechanisms for ensuring long-term safety, such as recursive feasibility guarantees. Furthermore, as the system dimensionality increases and the feasible region becomes increasingly complex and nonconvex, a finite set of sampled trajectories may fail to contain any feasible solution. % under which the controller may fail to produce a safe control action. 
Finally, another key limitation of sampling-based MPC for safety is that, although individual sampled trajectories may each satisfy the safety constraints, the final control solution obtained by aggregating these samples can still violate safety during the inference step \cite{borquez2025dualguard}.

Another school of control methods for safety assurance is to employ \textit{safety certificate functions} \cite{wang2026handbook}, typically computed offline. These methods verify a control invariant set and subsequently synthesize an online feedback controller that guarantees the system remains within this set, thereby satisfying the prescribed safety constraints. Hamilton-Jacobi reachability (HJR) \cite{bansal2017hamilton} and control barrier functions (CBFs) \cite{Ames2019} are prominent methodologies for designing safety certificates. However, both approaches suffer from poor scalability with respect to system dimensionality and the complexity of the dynamics, making them most suitable for reduced-order models. For high-dimensional, full-order robotic systems, safety is often achieved through robust formulations and layered control architectures \cite{CohenARC24}; however, this approach inherently entails additional conservatism in the resulting closed-loop system due to margins addressing model errors.

\begin{figure*}[!t]
    \centering
    \includegraphics[width=0.99\textwidth]{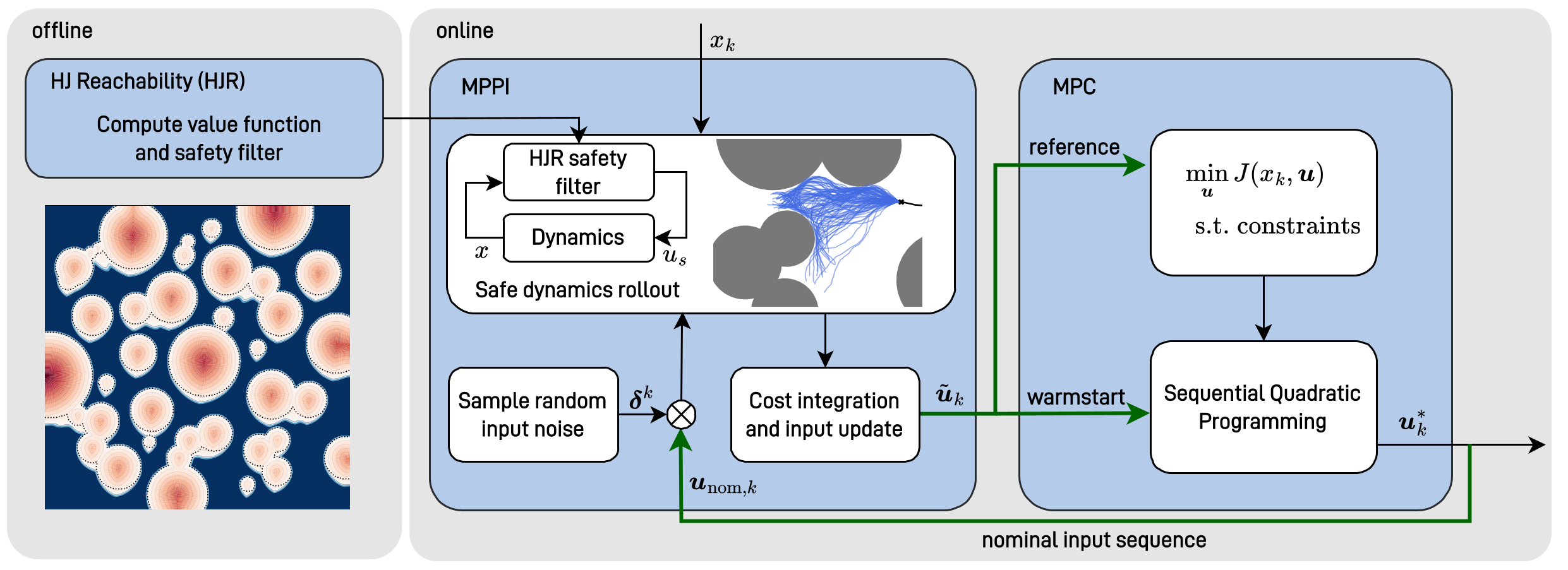}
    \vspace{-0.7em}
    \caption{Overview of \ourmethod. An offline-computed HJ reachability value function and safety filter are used in the MPPI trajectory sampling. The three proposed interactions of MPPI with sequential quadratic programming (SQP) are depicted with green arrows: 1) using the MPPI solution as a reference in the MPC optimization problem, 2) as a warmstart for the SQP solver, and 3) using the SQP optimal trajectory as a nominal input in MPPI.}
    \label{fig:blockdiagram}
    \vspace{-0.5cm}
\end{figure*}

Motivated by the limitations of existing approaches, this work proposes a new safe MPC framework, \ourmethod, that combines three key components: (1) an offline-computed HJR safety certificate based on a reduced-order model, (2) an MPPI that generates an initial solution candidate by sampling trajectories guided by the safety certificate, and (3) a small number of SQP iterations that refine the candidate solution to satisfy the constraints of the full-order system dynamics.

Our key observation is that the inference step in MPPI, which aggregates sampled control sequences into a single control solution, is an \textit{unconstrained} weighted least-squares optimization. Consequently, even if all sampled trajectories individually satisfy the safety constraints, the aggregated solution is not guaranteed to remain safe. Previous works \cite{yin2023shield, borquez2025dualguard} address this issue by introducing a post-processing safety filter that projects the aggregated control onto the safe set. In contrast, our key idea is to unify control aggregation and safety enforcement into a single optimization problem. Specifically, we formulate the inference step as a \textit{constrained} weighted least-squares problem within the SQP framework of nonlinear MPC.
Below, we detail how the proposed framework overcomes the individual limitations of existing SQP-based MPC, MPPI, and safety certificate methods.

\textit{\ourmethod~vs. SQP \& Sampling-based MPC:}
HJR-guided MPPI generates an initial solution that closely approximates the long-horizon optimal solution while accounting for safety. This warm start substantially improves the convergence and robustness of SQP. Unlike conventional MPPI, the proposed SQP formulation incorporates safety constraints directly into the control aggregation step, ensuring that the final control solution satisfies the desired safety requirements. Furthermore, the control sequence refined by SQP is used as the nominal input sequence for the MPPI at the next sampling instant, establishing a bidirectional coupling between MPPI sampling and SQP refinement. This interconnection significantly reduces the likelihood that either MPPI fails to sample a feasible trajectory or SQP fails to converge to a feasible solution within the limited number of optimization iterations.

\textit{\ourmethod~vs. HJR:} While conventional numerical HJR cannot be applied to higher-dimensional systems without neural network approximations, MPPI and MPC can impose hard safety constraints based on full-order dynamics.

\subsection{Related Works}

\paragraph{Predictive Control for Safety Assurance} MPC has long been regarded as one of the principal frameworks for safety due to its ability to explicitly incorporate state and input constraints into the online optimization. In classical MPC, safety is typically established through \emph{recursive feasibility}, ensuring that the optimization problem remains feasible at every sampling time by designing terminal constraint sets that are control invariant. However, constructing such terminal sets is often challenging, particularly for complex nonlinear systems and safety constraints. Moreover, the resulting nonlinear optimization can be computationally demanding and may fail to converge to a feasible solution in real-time. Predictive safety filters (PSF) \cite{wabersich2021predictive} have extended MPC to a supervisory safety layer that minimally modifies a nominal controller while maintaining safety through recursive feasibility. More recently, sampling-based MPC methods such as MPPI have emerged as alternatives for optimization-based MPC and PSF for safety \cite{borquez2025dualguard, feng2025bridging}. Nevertheless, because these methods rely on a finite set of sampled trajectories, they generally lack formal recursive feasibility guarantees \cite{yin2023shield}. Consequently, when feasible trajectories are not sampled or the final aggregated control input does not preserve feasibility, long-term safety is not ensured.

\paragraph{Safety Certificates + Sampling-based MPC} To improve the safety of sampling-based MPC, various approaches are proposed which integrate safety certificates into the trajectory sampling process. The methods in \cite{tao2022control, yin2023shield, yin2024chance} integrate CBFs with MPPI; however, the employed CBFs are typically simple distance-based functions that do not certify control invariance. The work of \cite{borquez2025dualguard} integrates HJR least-restrictive safety filtering with MPPI. However, the approach is limited to low-dimensional system models because of the poor scalability of HJR. To improve scalability, recent methods have proposed usage of learning-based certificates, including neural CBFs \cite{yin2025safe} and neural approximations of HJR value functions \cite{feng2025bridging}. While these approaches can be applied to higher-dimensional systems, formally verifying the safety guarantees of learned certificates remains an open challenge \cite{lin2024verification}. More recently, \cite{mizuta2026cfmmppi} uses CBFs to encourage safety during the training of a generative model that produces the proposal distribution for MPPI.

% \paragraph{Predictive Control for Safety Assurance and Safety Filtering}
% \begin{itemize}
%     \item In classical MPC - recursive feasibility as main mechanism for safety. However, hard to design terminal constraint, and for complex nonlinear systems, MPC becomes generally nonlinear optimization, and often faces infeasibility, sacrificing safety.
%     \item MPC has been extended to Predictive Safety Filtering Schemes.
%     \item Sampling MPC recently emerged as convenient safety filtering schemes, circumventing the need of solving the optimization directly. However, with sparse samples, no hard guarantee for recursive feasibility, and thus safety.
% \end{itemize}

% Safety Certificates
% \begin{itemize}
%     \item Set Invariance baked into the scalar function.
%     \item Offline computed, hard to hand-design or learn. 
%     \item Usually do not scale well with state dimension and complex nonlinear dynamics.
% \end{itemize}

% CBF - use distance functions as CBFs
% \begin{itemize}
%     \item \cite{tao2022control}
%     \item Shield MPPI \cite{yin2023shield} and its extensions \cite{yin2024chance}
% \end{itemize}

% CBF - Neural 
% \begin{itemize}
%     \item Neural CBF + Novel sampling technique \cite{yin2025safe}
%     \item CBF used to guide Conditional Flow Matching generative model training, which produce prior distribution for MPPI \cite{mizuta2026cfmmppi}, no hard safety guarantee.
% \end{itemize}

% Reachability
% \begin{itemize}
% \item \cite{borquez2025dualguard}
% \end{itemize}

% \paragraph{Misc. (Temporary)}
% \cite{reiter2025ac4mpc}

\subsection{Notation}
We denote a sequence of vectors $x_i$ as $\bm{x}\coloneqq\{x_i\}_{i=0}^N$, where $N$ is the length of the sequence. Optimization variables at time step $k$ predicting $i$ steps ahead are denoted as $x_{i|k}$ and the corresponding sequence of optimization variables is denoted as $\bm{x}_k$. An optimal solution of a given optimization problem with variables $x_{i|k}$ is denoted as $x^*_{i|k}$. %The $j$-th iterate of finding a solution at time step $k$ is denoted as $\bm{x}_{j|k}$.

\section{Preliminaries}
\subsection{Problem Formulation}
We consider discrete-time dynamical systems \vspace{-0.5em}
\begin{equation}
    x_{k+1} = f(x_k, u_k), \vspace{-0.5em}
\end{equation}
with state $x_k \in \mathbb{R}^n$, input $u_k\in\mathbb{R}^m$ and discrete-time dynamics $f:\mathbb{R}^n \times \mathbb{R}^m \rightarrow \mathbb{R}^n$. The system is subject to state and input constraints, such that $\forall k \geq 0$ it holds that
\begin{align*}
    x_k \in \mathcal{X} \coloneqq \{x\in\mathbb{R}^n \mid c_x(x) \leq 0\}, \quad u_k \in \mathcal{U},
\end{align*}
where $c_x$ is the signed distance function of the state constraint set $\mathcal{X}$. The input constraints $\mathcal{U}$ capture physical actuator limitations, while the state constraints encode safety-critical considerations such as obstacles or velocity limits. 

The aim is to ensure that the constraints of the system are satisfied during operation while minimizing a performance objective of the form \vspace{-0.75em}
\begin{equation}
    \mathcal{L}(\bm{x}, \bm{u})\coloneqq \sum_{k=0}^\infty l(x_k, u_k), \vspace{-0.5em}
\end{equation}
where $l:\mathbb{R}^n \times \mathbb{R}^m \rightarrow\mathbb{R}$ is a stage cost incurred at every time step. This results in the following optimal control problem: \vspace{-0.5em}
\begin{subequations} \label{eq:OCP}
\begin{align}
    \min_{\bm{u}_0} \; & \sum_{i=0}^\infty l(x_{i |0}, u_{i|0}) \\
    \mathrm{s.t.} \; & \forall i \geq 0, \nonumber \\
    & x_{0|0} = x_0, \\
    & x_{i+1 | 0} = f(x_{i|0}, u_{i|0}), \\
    & x_i \in \mathcal{X}, \\
    & u_i \in \mathcal{U}. \vspace{-0.5em}
\end{align} 
\end{subequations}
While such optimal control problems are generally intractable to solve, we aim to propose an approximate solution in this paper that satisfies the system constraints during operation.

\subsection{Hamilton-Jacobi Reachability} \label{sec:HJR}
We leverage Hamilton-Jacobi reachability (HJR) analysis to compute the maximal control invariant set and the safe optimal controller that maintains the system within the invariant set. Given the constraint set $\mathcal{X}=\{x\in\mathbb{R}^n \mid c_x(x)\leq 0\}$, a value function given as \vspace{-0.5em}
\begin{equation*}
    V(x) := \min_{\bm{u}\in\mathcal{U}} \max_{k\in \mathbb{N}} c_x(x_k), \vspace{-0.5em}
\end{equation*}
characterizes the maximal control invariant set in $\mathcal{X}$ as its zero-sublevel set, $\mathcal{S}:=\{x\in\mathbb{R}^n \mid V(x)\leq 0\}$ \cite{fisac2019bridging}, which we denote as the \textit{maximal safe set}. Moreover, the value function induces the \textit{safe optimal controller} $u_{\mathrm{s}}(x)$:

\;
\vspace{-0.5em}
\begin{equation*}
    u_{\mathrm{s}}(x) := \arg\min_{u\in\mathcal{U}} V(f(x, u)), \vspace{-0.5em}
\end{equation*}
which can be used to ensure that the system remains safe if $V(x)\leq 0$ \cite{borquez2025dualguard}.

\subsection{Model Predictive Path Integral Control} \label{sec:MPPI}
Model predictive path integral (MPPI) \cite{williams2017model, williams2017information} is a receding horizon approach to obtain an approximate solution to the optimal control problem in~\eqref{eq:OCP} for unconstrained systems. The optimal control problem~\eqref{eq:OCP} is approximated with an optimization problem of finite horizon $N$ through a terminal cost, i.e.,
\begin{subequations} \label{eq:MPPIOCP}
\begin{align}
    \min_{\bm{u}_0} \; &  \sum_{i=0}^{N-1} l(x_{i |k}, u_{i|k}) + m(x_{N|k}) \label{eq:MPPIOCP-obj}\\
    \mathrm{s.t.} \; & \forall i = 0,\dots,N-1 \nonumber \\
    & x_{0|k} = x_k, \\
    & x_{i+1 | k} = f(x_{i|k},u_{i|k}), 
\end{align}
\end{subequations}
where the terminal cost $m:\mathbb{R}^n\rightarrow\mathbb{R}$ aims to capture the cost-to-go of the infinite horizon tail. It is shown in \cite{williams2017model} that~\eqref{eq:MPPIOCP} can be approximately solved by rolling out $M$ random input disturbances, aggregating their cost and performing a weighted average update to the input sequences. Each randomly sampled trajectory $\tau^j_k$, consisting of a nominal input sequence $\bm{u}_{\mathrm{nom,k}}$ and random Gaussian perturbations $\delta\bm{u}_k^j$  along with their corresponding state sequence $\bm{x}_k^j$, results in a cost \vspace{-0.5em}
\begin{equation*}
    S(\tau^j_k) \coloneqq m(x_{N|k}^j) + \sum_{i=0}^{N-1} l(x_{i|k}^j, u_{i|k}^j)
\end{equation*}

This results in the following update rule
\begin{equation} \label{eq:MPPIupdate}
    \tilde{\bm{u}}(x_k) \approx \bm{u}_{\mathrm{nom,k}} + \frac{\sum_{j=1}^M \exp\big(-\nicefrac{1}{\lambda}S(\tau^j_k)\big)\delta \bm{u}^j_k}{\sum_{j=1}^M \exp\big(-\nicefrac{1}{\lambda}S(\tau^j_k)\big)}, 
\end{equation}
where $\lambda$ is a user-defined temperature that allows tuning the effective number of trajectories contributing to the weighted average. 

A crucial observation is that the MPPI update in \eqref{eq:MPPIupdate} is precisely the solution of an \textit{unconstrained} weighted least-squares problem.
\begin{proposition}
    The MPPI update \eqref{eq:MPPIupdate} is the solution of the optimization problem
    \begin{equation}
    \min_{\bm{u}_k}  \sum_{j=1}^{M} \exp\big(-\nicefrac{1}{\lambda}S(\tau^j_k)\big) \Vert \bm{u}_k - (\bm{u}_{\textup{nom, k}} + \delta \bm{u}^j_k) \Vert^2 \label{eq:LS-MPPI}.
\end{equation}
\end{proposition}
\begin{proof}
    The optimal solution of the unconstrained optimization problem minimizing~\eqref{eq:LS-MPPI} can be derived by setting
    \begin{align*}
        0 = & \nabla_{\bm{u}_k} \sum_{j=1}^{M} \exp\big(-\nicefrac{1}{\lambda}S(\tau^j_k)\big) \Vert \bm{u}_k - (\bm{u}_{\textup{nom, k}} + \delta \bm{u}^j_k) \Vert^2 \\
        = & \; 2 \sum_{j=1}^{M} \exp\big(-\nicefrac{1}{\lambda}S(\tau^j_k)\big) (\bm{u}_k - (\bm{u}_{\textup{nom, k}} + \delta \bm{u}^j_k))
    \end{align*}
    Solving for $\bm{u}_k$ results in the unconstrained optimal solution~\eqref{eq:MPPIupdate}.
\end{proof}
As such, the final closed-loop control is inherently agnostic to system constraints and therefore lacks safety guarantees.

\subsection{HJR-guided MPPI (Dualguard-MPPI)}
\label{sec:hjr-mppi}

One way to improve the safety of MPPI is to bias the trajectory sampling process using safety certificates. Existing approaches achieve this by either augmenting the trajectory cost in \eqref{eq:MPPIOCP-obj} with barrier or value-function penalties, or by rejecting unsafe samples through safety filtering \cite{yin2023shield, borquez2025dualguard}. Neither of which guarantees closed-loop constraint satisfaction as discussed in \cite{borquez2025dualguard}; thus, we leverage the HJR-based safety filtering to generate safe trajectory samples as proposed in \cite{borquez2025dualguard}, named \textit{Dualguard-MPPI}. Specifically, for each sampled input sequence $\bm{u}_k^j$, at time step $i$, if the state at the next time step $x_{i+1|k}^j$ is not within $\mathcal{S}$, we override $u_{i|k}^j$ with the safe optimal control input $u_{\mathrm{s}}(x_{i|k}^j)$, making sure that the trajectory stays within the maximal invariant set.

The main limitation of this approach is the poor scalability of the computation of the HJR value function $V$. Due to the curse of dimensionality \cite{bansal2017hamilton}, we can only use a reduced-order model of $f$ for its computation, which results in errors in the computed maximal safe set $\mathcal{S}$ and the safe controller $u_{\mathrm{s}}$. Thus, the safety for the full-order system is compromised with this approach itself, as we will see in our results.

% Another limitation incorporating safety filtering into the MPPI update rule in is that while such filtering can ensure the closed-loop constraint satisfaction, they are typically unaware of the original primary control objective in \eqref{eq:MPPIOCP}. Thus, the safe control input $u_{\mathrm{s}}(x_{i|k}^j)$, overriding the original sampled input can result in substantial drop of performance. 

\subsection{Model Predictive Control} \label{sec:MPC}
Another receding horizon approach to approximate~\eqref{eq:OCP} is model predictive control (MPC), see \cite{rawlings2017} for an overview. In contrast to MPPI, which proposes a sampling-based solution approach to~\eqref{eq:MPPIOCP}, MPC leverages numerical optimization in order to solve constrained optimization problems of the form \vspace{-0.75em}
\begin{subequations} \label{eq:MPCOCP} 
\begin{align}
    \min_{\bm{u}_0} \; &  \sum_{i=0}^{N-1} l(x_{i |k}, u_{i|k}) + m(x_{N|k})\\
    \mathrm{s.t.} \; & \forall i = 0,\dots,N-1 \nonumber \\
    & x_{0|k} = x_k, \\
    & x_{i+1 | k} = f(x_{i|k}, u_{i|k}), \\
    & x_{i|k} \in \mathcal{X}, \\
    & u_{i|k} \in \mathcal{U}, \\
    & x_{N|k} \in \mathcal{X}_f. \label{eq:MPCterm}
\end{align}
\end{subequations}
Similarly to MPPI, the terminal cost $m(x_{N|k})$ is employed to approximate the cost to go, while an invariant terminal set $\mathcal{X}_f$, as e.g. described in Section~\ref{sec:HJR}, is used in~\eqref{eq:MPCterm} to ensure that the optimization problem~\eqref{eq:MPCOCP} is recursively feasible. 

In order to solve~\eqref{eq:MPCOCP} online in real-time, we consider sequential quadratic programming (SQP) approaches such as described in \cite{boggs1995sequential}. SQP is an iterative optimization approach where the problem~\eqref{eq:MPCOCP} is approximated as a quadratic program at the current solution iterate $\bar{\bm{x}}_{j|k}$ and $ \bar{\bm{u}}_{j|k}$. The solution of the quadratic program provides the next update step, until the iterates converge to a local KKT point. 

Due to the local quadratic approximations, the solution quality of SQP solvers heavily depends on the initial warmstart of the solution iterates $\bar{\bm{x}}_{0|k}$ and $\bar{\bm{u}}_{0|k}$. For MPC methods, the initial iterates are typically a time shifted solution obtained at the previous time step $k-1$. This can result in closed-loop solutions which are stuck in local optima for non-convex problems and may often fail to find feasible solutions. In this work, we aim to leverage the exploration capabilities of MPPI samples to resolve local optima issues.

\section{METHOD}\label{sec:method}

In this section, we introduce our method, \textbf{Reachability-guided SQP-guarded model predictive path integral (\ourmethod)} as illustrated in Figure~\ref{fig:blockdiagram}. The aim is to augment HJR-guided MPPI as described in Section~\ref{sec:hjr-mppi} with constraint satisfaction guarantees. We achieve this through a best of both worlds approach, where MPPI allows escaping local optima through its sampling-based solution approach, the potentially unsafe solution of which we then refine using SQP to achieve a high performance safe control framework for non-convex optimal control problems. We propose three interconnections between MPPI and MPC through SQP: (A) recovering the MPPI solution through a reference tracking MPC formulation, (B) using the MPPI solution as a warmstart for SQP iterations, and (C) lastly feeding back the converged SQP solution as the nominal input sequence in MPPI at the next sampling time.

\subsection{MPPI Solution via MPC}
\label{sec:ref}

We first consider the MPC formulation which recovers the optimal solution of MPPI~\eqref{eq:MPPIupdate} if it satisfies the system constraints and modifies it as little as possible if it is unsafe. Such an approach is similar to predictive safety filters as described in \cite{wabersich2021predictive}, although we consider the weighted average solution of MPPI over multiple time steps: \vspace{-0.5em}
\begin{subequations} \label{eq:LSMPC}
\begin{align}
    \min_{\bm{u}_k} & \sum_{j=1}^{M} \exp\big(-\nicefrac{1}{\lambda}S(\tau^j_k)\big) \Vert \bm{u}_k - (\bm{u}_{\textup{nom, k}} + \delta \bm{u}^j_k) \Vert^2 \label{eq:LSMPCcost}\\
    \mathrm{s.t.} \; & \forall i = 0,\dots,N-1 \nonumber \\
    & x_{0|k} = x_k, \label{eq:LSMPCinit} \\
    & x_{i+1 | k} = f(x_{i|k}, u_{i|k}), \label{eq:LSMPCdyn}\\
    & x_{i|k} \in \mathcal{X}, \label{eq:LSMPCstate}\\
    & u_{i|k} \in \mathcal{U}, \label{eq:LSMPCinput}\\
    & x_{N|k} \in \mathcal{X}_f, \label{eq:LSMPCterm}
\end{align}
\end{subequations}
The objective of the MPC is exactly identical to \eqref{eq:LS-MPPI}; however, the crucial difference is the incorporation of the constraints \eqref{eq:LSMPCstate}-\eqref{eq:LSMPCterm} for safety guarantees. If the inputs lead the system to an unsafe region, the optimal solution of~\eqref{eq:LSMPC} will track the inputs as closely as possible such that state and input constraints are satisfied. Thus, the formulation in \eqref{eq:LSMPC} can be considered a \textit{constrained version} of the weighted least-squares MPPI update rule in \eqref{eq:MPPIupdate} for safety. Clearly, if the MPC does not modify the MPPI output if it already satisfies the system constraints.

For many systems, rather than track an input sequence, it is desirable to track the resulting states instead. In such cases, instead of tracking the MPPI input sequence, we consider tracking the corresponding optimal state sequence:
\vspace{0.5em}
\begin{corollary}
    If the system dynamics $f(x_k, u_k)$ are bijective with respect to $u_k$, the optimal solution of 
    \begin{subequations} \label{eq:LSMPC_state}
\begin{align}
    \min_{\bm{u}_k} \;\; & \Vert \bm{x}_k - \tilde{\bm{x}}_k \Vert^2 \label{eq:LSMPC_state_cost}\\
    \mathrm{s.t.} \; & \forall i = 0,\dots,N-1 \nonumber \\
    & x_{0|k} = x_k, \\
    & x_{i+1 | k} = f(x_{i|k}, u_{i|k}), \\
    & x_{i|k} \in \mathcal{X}, \\
    & u_{i|k} \in \mathcal{U}, \\
    & x_{N|k} \in \mathcal{X}_f,
\end{align}
\end{subequations} 
where $\tilde{\bm{x}}_k$ denotes the state sequence corresponding to~\eqref{eq:MPPIupdate}, is given by $\tilde{\bm{u}}_k$ if $\tilde{\bm{u}}_k$ and $\tilde{\bm{x}}_k$ satisfy \textup{(\ref{eq:LSMPC}d--f)}.
\end{corollary}
\begin{proof}
    The claim follows directly from bijectivity of the dynamics, as the unconstrained minimizer of~\eqref{eq:LSMPCstate}, $\tilde{\bm{x}}_k$, can only be obtained through the input sequence $\tilde{\bm{u}}_k$ in~\eqref{eq:MPPIupdate}.
    % \ncnote{review here: check if we need bijectivity or injectivity of the map F where $\tilde{\bm{x}}_k := F\left($\tilde{\bm{u}}_k$\right)$ at $\tilde{\bm{u}}_k$ + feasibility}
\end{proof}
Thereby, we can recover the MPPI solution not just via an input matching approach, but also through a state trajectory tracking formulation.

\begin{remark}[Hierarchical Structure]
For high-dimensional systems, HJR-guided MPPI can become computationally intractable as noted in \ref{sec:hjr-mppi}. Furthermore, with increasing input dimensions, the likelihood of sampling meaningful random input sequences decreases. Our approach naturally allows decoupling the MPPI and SQP approach into a hierarchical framework. With a reduced order model,
% \begin{equation*}
%     x^{\mathrm{red}}_{k+1} = f^{\mathrm{red}}(x^{\mathrm{red}}_k, u^{\mathrm{red}}_k),
% \end{equation*}
we can use HJR-guided MPPI as a higher-level planner, providing a reference $\tilde{\bm{x}}_k^{\mathrm{red}}$. This can then be used in the tracking formulation in \eqref{eq:LSMPC_state}, allowing the online optimization to obtain a control sequence for the full-order model. We show that this hierarchical structure can provide competitive performance in an autonomous racing example in Section~\ref{sec:simulations}.    
\end{remark}

\subsection{MPPI as Warmstart for SQP} 
\label{sec:ws}
While the MPPI update rule can be recovered through a predictive optimization problem as shown in Section~\ref{sec:ref}, the suboptimality of the MPPI solution grows second-order with the input sampling uncertainty \cite{homburger2025optimality}. Thereby, the MPPI solution may lose accuracy if the samples are under large noise perturbations, or the samples are too sparse in high-dimensional state space. In order to remedy these effects, we propose refining the MPPI solution by using it as the initial iterate in a sequential quadratic formulation. By minimizing the same objective in MPPI as in the following SQP, the MPPI solution can be guided towards the closest local KKT point. The tracking cost in \eqref{eq:LSMPC_state} can then additionally be used as a regularizer to ensure local optimization of the objective in environments with multiple local optima such as, e.g., the cluttered environment simulation in Section~\ref{sec:simulations}.

\vspace{0.5em}
\begin{remark}
    In \cite{fazlyab2026model}, it is shown that the MPPI update corresponds to a preconditioned gradient step on the performance objective. By using the MPPI solution as a warmstart to an SQP, we apply a gradient step followed by SQP updates (which correspond to Gauss-Newton steps in the unconstrained setting based on the Hessian approximation) with respect to the control objective. Investigating the convergence properties of this approach is an interesting research direction for future work.
\end{remark}

\subsection{SQP Solution as MPPI Nominal Sequence}
\label{sec:nom}
The MPPI update rule in \eqref{eq:MPPIupdate} requires a nominal input sequence to which the weighted average of the randomly sampled input trajectories is added to. As this nominal input sequence biases the distribution of random samples for exploration, this nominal sequence can play a large role in the closed-loop performance of MPPI. The third element of our approach is to use the time-shifted optimal solution of the SQP as the nominal solution of MPPI at the next time step. This interaction, together with the first two elements, closes the loop between MPPI and SQP.

The overall \ourmethod\ control architecture, which is composed of the three types of interactions between HJR-guided MPPI and SQP, is illustrated in Figure~\ref{fig:blockdiagram}. % and detailed in Algorithm~\ref{alg:\ourmethod}
% \begin{algorithm}[]
%     \caption{} \label{alg:\ourmethod}
%     \begin{algorithmic}[1] 
%         \State 
%     \end{algorithmic}
% \end{algorithm}
% {Hierarchical Structure}
% For high dimensional systems, the Hamilton-Jacobi reachability approach used to filter the MPPI sample rollouts can become computationally intractable \cite{bansal2017hamilton}. Furthermore, with increasing input dimensions, the likelihood of sampling meaningful random input sequences decreases. The proposed \ourmethod approach naturally allows decoupling the MPPI and SQP approach into a hierarchical framework. For a reduced order model 
% \begin{equation*}
%     x^{\mathrm{red}}_{k+1} = f^{\mathrm{red}}(x^{\mathrm{red}}_k, u^{\mathrm{red}}_k),
% \end{equation*}
% we can use MPPI as a higher-level planner, providing a reference $\tilde{\bm{x}}_k^{\mathrm{red}}$ which can be used in the tracking formulation in \eqref{eq:LSMPC_state}, allowing the online optimization to obtain a control sequence for the full-order model. We show that this hierarchical structure can provide competitive performance in an autonomous racing example in Section~\ref{sec:simulations}.
\section{SIMULATIONS} \label{sec:simulations}
In this section, we evaluate \ourmethod~on a vehicle navigating a cluttered environment and in an autonomous racing task with random cluttered obstacles on the track.

\subsection{System Models \& Offline Computation}\label{sec:dynamics}
We consider two vehicle models as considered in \cite{carron2023chronos}. First, we consider a kinematic bicycle model, which captures the vehicle's nonholonomic motion, for the navigation task and as the reduced-order model for the racing task. Secondly, we consider a dynamic bicycle model with Pacejka tire force models which captures realistic tire slip during high-speed racing maneuvers.

The kinematic bicycle model is given as 
\begin{align} 
\dot{p}_x(t)\!=\!v(t) & \cos\left(\psi(t) \!+\!\beta(t)\right); \;\; \dot{p}_y(t) \!=\! v(t)\sin\left(\psi(t)\!+\!\beta(t)\right) \nonumber \\
& \dot{\psi}(t) = \frac{v(t)}{l_r}\sin(\beta(t)); \quad
\dot{v}(t) = a(t), \label{eq:kinematic}
\end{align}
where $\beta(t)\!=\!\arctan\!\left(\frac{l_r}{l_r + l_f}\tan(\delta(t))\!\right)$ is the slip angle and the control inputs consist of the steering angle $\delta$ and the acceleration $a$. 

We compute the HJR value functions for the cluttered navigation environment and the race track using the \texttt{hj\_reachability\_torch} library~\cite{choi2025datadriven}. Computing the value function for each environment instance takes approximately 20 minutes on an NVIDIA GeForce RTX 5090 GPU.

The full-order model with Pacejka tire forces is detailed in \cite{carron2023chronos}. We additionally augment the model with first-order actuator dynamics for the torque and the steer. This results in an eight-dimensional state, with the torque and steer rate commands serving as control inputs.

\subsection{Baselines \& Metrics}

As baselines, we compare our approach to plain MPPI, HJR-guided MPPI (Dualguard-MPPI) \cite{borquez2025dualguard}, and SQP-based MPC in \eqref{eq:MPCOCP}. 
%For plain MPPI, we include a collision penalty cost based on the signed distance function of the constraint set $c_{\mathrm{x}}(x)$. 
For plain MPPI and Dualguard-MPPI, we included a penalty based on the HJR value function as done in the original Dualguard work \cite{borquez2025dualguard}. Computing the HJR value function for the eight-dimensional full-order model is computationally prohibitive with current computational resources. For the racing task, Dualguard-MPPI is not included in the reported results since its reduced-order-based implementation performed poorly under full-order simulation compared to other methods. For MPCs, the constraints are implemented as soft constraints, which are solved with the SQP solver \cite{verschueren2022acados}. If a feasible solution is not found, we use the solution from the previous control loop as a backup solution. The collision rate is given as the percentage of episodes with at least one collision with the environment. The success rate is defined as the percentage of collision-free episodes reaching the goal in the cluttered environment or completing a lap in the racing task before a 20s timeout. 

\subsection{Experiment 1: Navigation in Cluttered Environment}
The first simulation study considers the kinematic bicycle model in a planar navigation setting. We use the cluttered environment from \cite{borquez2025dualguard} as depicted in Figure~\ref{fig:cluttered}. For 100 randomly generated start and goal points, the goal is to minimize the following cost function, consisting of a tracking cost with respect to the goal position $p_{\mathrm{goal}}$ as well as a penalty on steering and acceleration of the system,
\begin{equation*}
    \mathcal{L}(\bm{x},\bm{u}) = \sum_{k=0}^\infty \Vert p_{x, k} - p_{x,\mathrm{goal}}\Vert^2 + \Vert p_{y, k} - p_{y,\mathrm{goal}}\Vert^2 + 2\Vert u_k \Vert^2.
\end{equation*}
Each episode is terminated if the goal is reached within $0.1$m. For all MPPI implementations, we use 100 samples and a prediction horizon of 50 time steps. 
\begin{figure}[]
    \centering
    \includegraphics[width=0.49\linewidth]{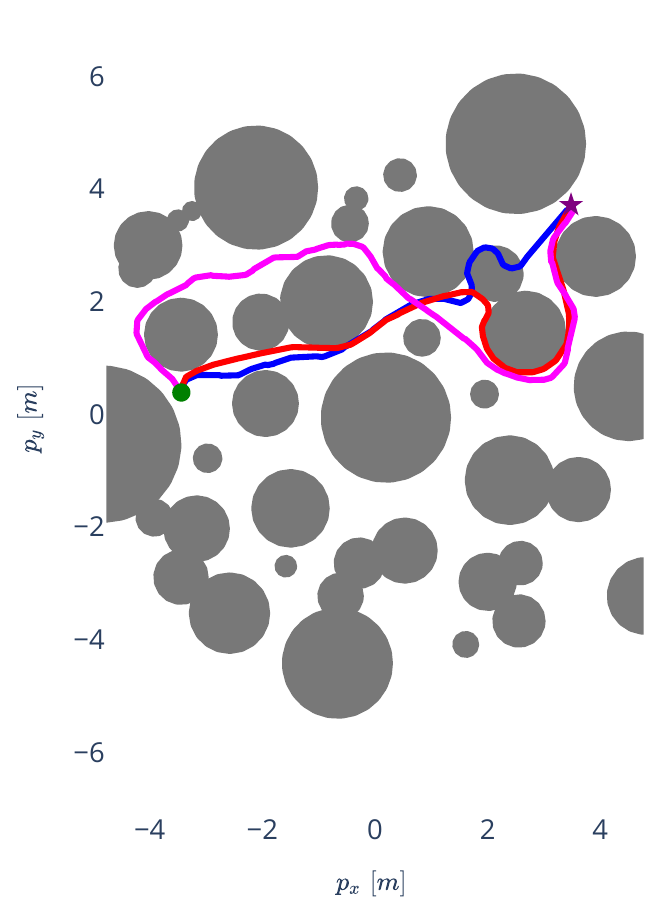}
    \includegraphics[width=0.49\linewidth]{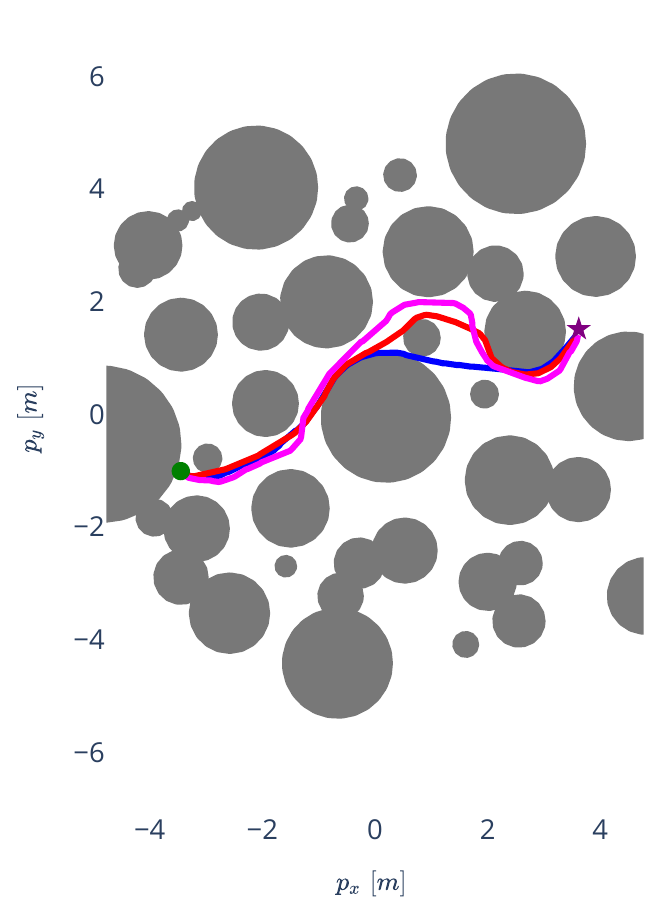}

    \includegraphics[width=0.9\linewidth]{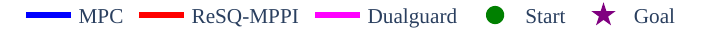}
    \vspace{-0.2cm}
    \caption{Two examples of simulated episodes. MPC suffers from local optima, resulting in collisions with the obstacles, while the HJR-based safety filter can lead to performance agnostic actions (left). The proposed SQP can improve performance compared of obtained MPPI solutions.}
    \label{fig:cluttered}
    \vspace{-0.5cm}
\end{figure}
\begin{table}[]
    \vspace{0.4cm}
    \centering
    \caption{Cluttered environment comparison}
    \begin{tabular}{cccc}
         &Avg. Cost&Success [\%]&Collisions [\%]\\
       \hline MPC (SQP) &  1.039 $\pm$ 0.089 & 26 & 71  \\
       \hline \makecell{Dualguard \\ MPPI \cite{borquez2025dualguard}} & 1.554 $\pm$ 0.124 & 91 & 0 \\
       \hline \makecell{\textbf{\ourmethod} \\ (ref. + ws. + nom.)} & 1.0 $\pm$ 0.102 & 96 & 0
    \end{tabular}
    \vspace{-0.5cm}
    \label{tab:cluttered}
\end{table}
The normalized, average incurred cost and standard error of jointly successful episodes of the proposed \ourmethod\ with the reference, warmstart and nominal interactions described in Section~\ref{sec:method} can be found in Table~\ref{tab:cluttered}. As the obstacles are directly integrated as non-convex quadratic constraints in the optimization problem, we observe that MPC struggles with local optima and numerical infeasibility, resulting in a large amount of collisions.
While Dualguard-MPPI remains safe throughout all episodes, the success rate is lower compared to our approach while incurring a substantially higher cost.
In comparison, we observe that \ourmethod\ alleviates the local optima issues of MPC by warmstarting the optimization problem with the MPPI solution, achieving the highest success rate while having no obstacle collisions, and achieving the lowest cost due to the additional optimization step by leveraging the strength of all three elements in Section~\ref{sec:method}. Two examples of simulated episodes are provided in Figure~\ref{fig:cluttered}.

\subsection{Experiment 2: Racing with obstacles}\label{sec:racing}
The second simulation study is performed on autonomous racing on a track with random obstacles. The aim is to achieve the lowest possible lap time, while avoiding collision. We generate 50 random configurations of obstacles, an example of which is depicted in Figure~\ref{fig:racing}. In order to approximate the time-optimal objective, we employ a contouring control formulation \cite{liniger2015optimization}. Instead of Cartesian position and global heading coordinates as described in \eqref{eq:kinematic}, a curvilinear parametrization is employed, capturing the progress $s$, lateral deviation $e$ and relative angle $\mu$ along the centerline of the considered track as derived similarly in \cite{liniger2015optimization}. as
\begin{align*}\label{eq:curvilinear}
    \dot{s}(t) & = \frac{v_x(t)\cos(\mu(t)) - v_y(t)\sin(\mu(t))}{1-e(t)\kappa(s(t))}, \\
    \dot{e}(t) &= v_x(t)\sin(\mu(t)) + v_y(t)\cos(\mu(t)), \\
    \dot{\mu}(t) &= \omega(t) - \kappa(s(t))\dot{s}(t), 
\end{align*}
where $\kappa(s)$ denotes the curvature of the centerline. The generated obstacles are implemented as quadratic in this curvilinear representation. 

The cost considered in the MPPI and MPC formulations is then given as
\begin{equation*}
    \mathcal{L}(\bm{x}_k,\bm{u}_k)\!=\! \sum_{i=0}^{N-1} q_1\Vert s_{i|k}\Vert^2 + q_2\Vert \mu_{i|k}\Vert^2 + q_3\Vert e_{i|k}\Vert^2 + r\Vert u_{i|k} \Vert^2.
\end{equation*}
%where $s$ is XX, and $e$ is XX.

We run simulations both for the kinematic model as well as the dynamic model with Pacejka tire forces as described in Section~\ref{sec:dynamics}. The lap times as well as obstacle collisions for the kinematic model are reported in Table~\ref{tab:kinematic} with 60 MPPI samples and a prediction horizon of 1s. In this setting, MPPI struggles to sample the bang-bang torque control required for a fast lap time in the racing application. This also impacts the \ourmethod\ controllers, where the reference is used in the cost formulation. We find that using the MPPI as warmstart and looping back the SQP solution as the nominal controller results in the fast lap times while keeping the system safe from collisions. 
An additional benefit of the proposed approach is that with the guidance of HJR certificate, we can significantly reduce the prediction horizon of the MPPI and SQP formulations, while preserving similar performance, as shown in Table~\ref{tab:prediction}.

For the eight-dimensional dynamic bicycle model, we leverage the flexibility of our approach, using the HJR value function of the reduced-order model as an additional cost in the full-order MPPI sampling. The lap times and failure rates are reported in Table~\ref{tab:dynamic}, where the proposed \ourmethod\ with all three interactions as proposed in Section~\ref{sec:method} demonstrates the fastest lap time while remaining safe across all episodes, failing only a single episode due to simulation timeout. In comparison, MPPI on its own struggles to find solutions due to the low number of samples and either collides or times out in all episodes. MPC, while fast lap times, incurs a lot of collisions in this setting.

% For the dynamic bicycle model, the number of states is too large to compute the HJR-based value function in a reasonable amount of time and we can therefore not use the safe policy to filter the MPPI dynamics rollout. This also renders the Dualguard \cite{borquez2025dualguard} baseline not feasible for this simulation. Instead, we leverage the flexibility of the proposed approach, using the value function of the reduced-order model as an additional cost in the full-order MPPI sampling. The lap times and failure rates are reported in Table~\ref{tab:prediction}, where the proposed \ourmethod\ with all three interactions as proposed in Section~\ref{sec:method} demonstrates the fastest lap time while remaining safe across all episodes, failing only a single episode due to simulation timeout. In comparison, MPPI on its own struggles to find solutions due to the low number of samples and either collides or times out in all episodes. MPC, while fast incurs a lot of collisions in this setting.

\begin{figure}[!t]
    \vspace{0.3cm}
    \centering
    \includegraphics[width=0.48\linewidth]{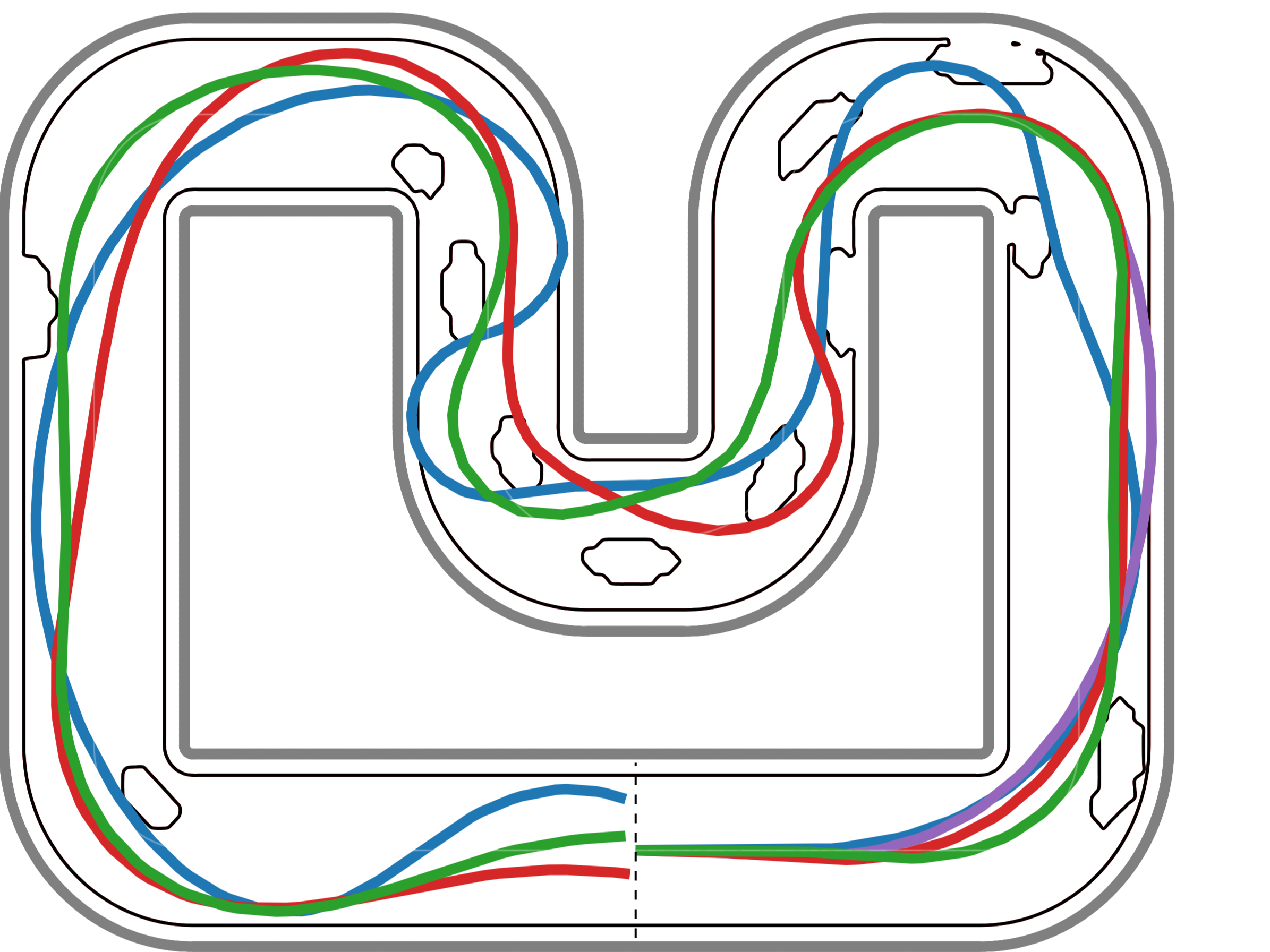}
    \includegraphics[width=0.48\linewidth]{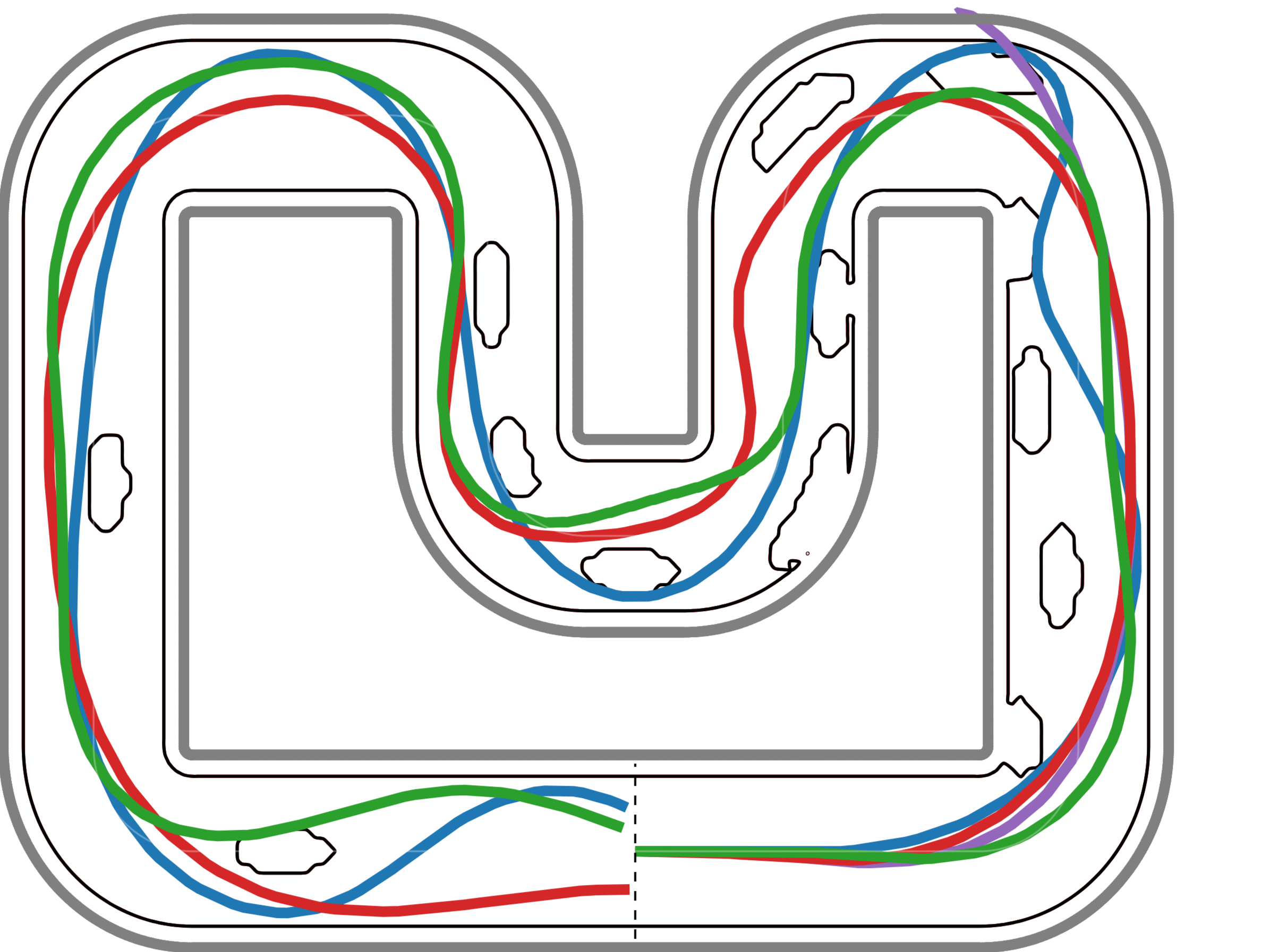} \\
    \vspace{0.2cm}
    \includegraphics[width=0.48\linewidth]{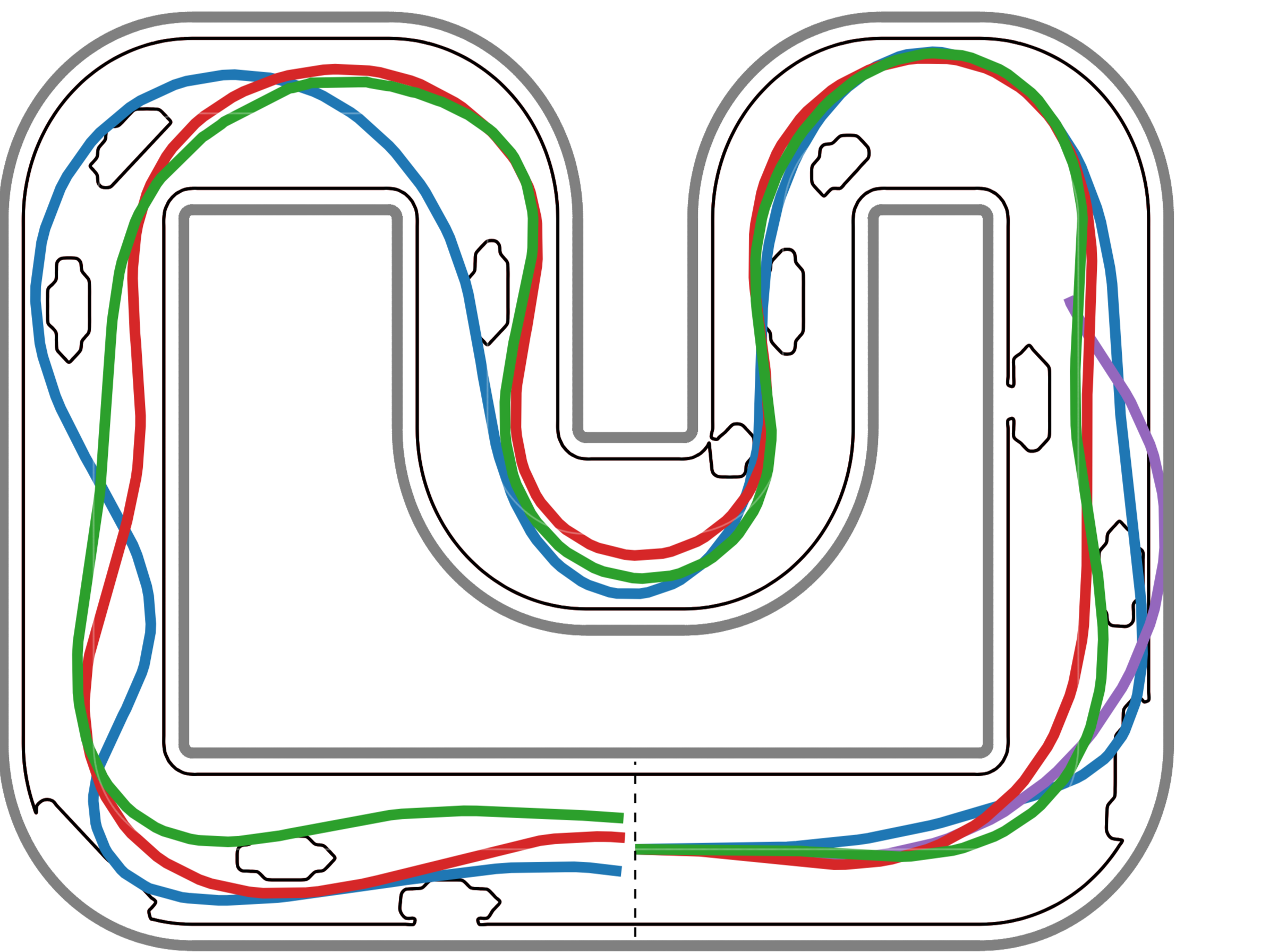}
    \includegraphics[width=0.48\linewidth]{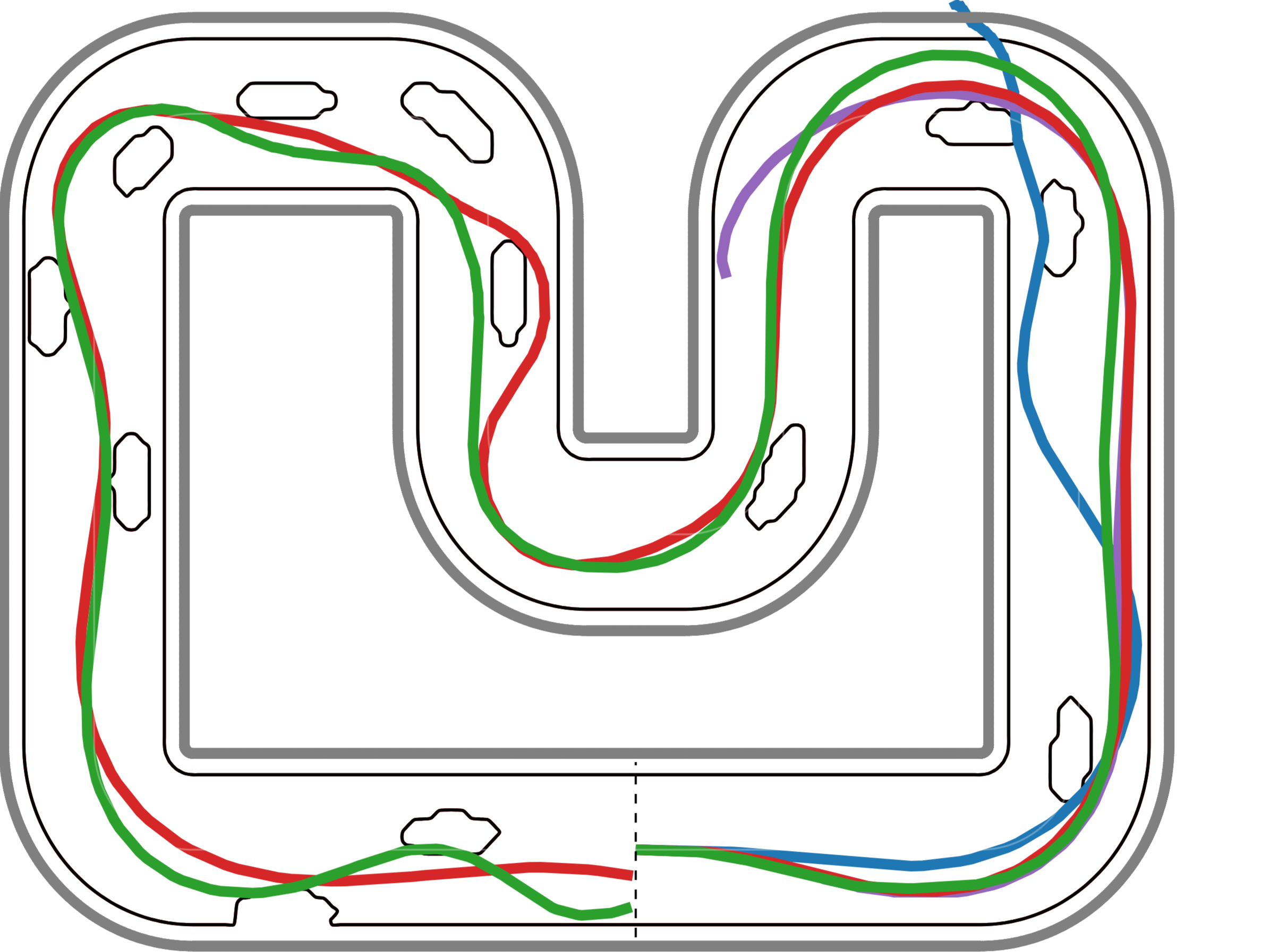} \\
    \includegraphics[width=0.98\linewidth]{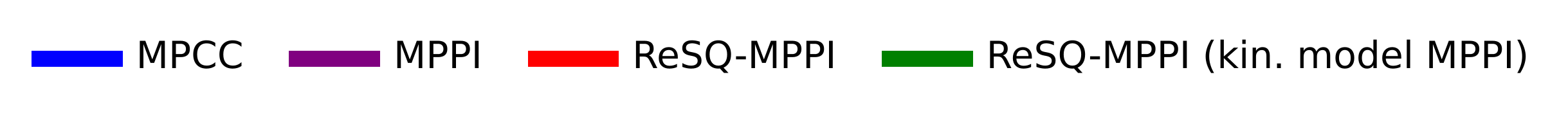}
    \vspace{-0.5cm}
    \caption{Closed-loop trajectories for the dynamic bicycle model with Pacejka tire forces on the considered track for four random obstacle configuration. The HJR value function level set computed for the kinematic model capturing the track boundaries and the obstacles is depicted (\textit{black}).}
    \label{fig:racing}
    \vspace{-0.1cm}
\end{figure}
\begin{table}[]
    \centering
    \caption{Kinematic Model Racing with Obstacles Comparison}
    \begin{tabular}{cccc}
     & Avg. Lap Time [s] & Success [\%] & Collision [\%] \\
     \hline MPC (SQP) & 4.274 $\pm$ 0.020 & 32 & 68 \\
     \hline Dualguard \cite{borquez2025dualguard} & 5.831 $\pm$ 0.210 & 96 & 4 \\
     \hline \makecell{\textbf{\ourmethod} \\ (ws. + nom.)} & 3.884 $\pm$ 0.012& 100 &
0    \end{tabular}
    \label{tab:kinematic}
    \vspace{-0.2cm}
\end{table}
\begin{table}[!t]
    \vspace{0.2cm}
    \centering
    \caption{Kinematic Model \ourmethod\ Prediction Horizon Study}
    \vspace{-0.2cm}
    \begin{tabular}{cccc}
        Pred. Hor. [s] & Avg. Lap Time [s] & Success [\%] & Collision [\%] \\
        \hline 1 & 3.867 $\pm$ 0.010 & 100 & 0 \\
        \hline0.5 & 3.851 $\pm$ 0.010 & 98 & 2 \\
        \hline0.4 & 3.870 $\pm$ 0.009 & 100 & 0 \\
        \hline0.3 & 4.126 $\pm$ 0.012 & 100 & 0 \\
        \hline0.2 & 7.144 $\pm$ 0.028 & 100 & 0 \\
        \hline0.1 & 14.996 $\pm$ 0.106 & 62 & 38
    \end{tabular}
    \label{tab:prediction}
    \vspace{-0.2cm}
\end{table}

\begin{table}[t]
    \centering
    \caption{Dynamic Pacejka Model Racing with Obstacles Comparison}
    \vspace{-0.2cm}
    \begin{tabular}{cccc}
         & Avg. Lap Time [s] & Succ. [\%] & Coll. [\%] \\
         \hline MPC (SQP) & 7.629 $\pm$ 0.124 & 42 & 58 \\
         \hline MPPI & - & 0 & 38 \\
         \hline \makecell{\textbf{\ourmethod} \\ (ref. + ws. + nom.)}  & 7.546 $\pm$ 0.198 & 98 & 0 \\
         % \hline \hline \makecell{\textbf{\ourmethod} \\ (kin. model MPPI)} & 8.668 $\pm$ 0.172 & 90 & 10 
    \end{tabular}
    \label{tab:dynamic}
    \vspace{-0.6cm}
\end{table}
% \begin{itemize}
%     \item racing application with obstacles; side benefit of reducing the prediction horizon
%     \item racing application results, discussion on reduced and high order model
% \end{itemize}

\subsection{Real-time implementation}
We test the proposed \ourmethod\ approach on a real-time simulation of the Chronos miniature race car in \cite{carron2023chronos} using ROS2 \cite{ros2} for one of the seeds in the racing simulation in Section~\ref{sec:racing}. We test the proposed \ourmethod\ approach with 40 samples with a prediction horizon of 40 time steps and 3 SQP iterations, running at 40Hz on a Intel i9-13900KS CPU. The average computation time of \ourmethod\ over 1000 time steps is 20.37ms, with a maximum time of 28.33ms. In comparison, MPPI alone with the HJR value function-based safety cost requires on average 8.58ms with a maximum time of 11.53ms and MPC with 3 SQP iterations averages 1.06ms with a maximum of 1.79ms. While MPPI and MPC both collide in the real-time simulation, \ourmethod\ has an average 5.9s lap time over 10 consecutive laps without collision.  

\section{CONCLUSIONS}
We propose \ourmethod, a control framework leveraging Hamilton-Jacobi reachability for safe reduced-order model-based sample generation in MPPI, whose optimal control sequence is enhanced for safety and performance through sequential quadratic programming iterations. We leverage the MPPI solution as a reference in a tracking MPC formulation and as a warmstart for the SQP iterations, feeding back the optimal SQP solution as the nominal MPPI input. We show that the proposed approach can alleviate local optimality issues in nonlinear MPC formulations as well as improve the safety and performance of MPPI in extensive simulation studies, including navigation in a cluttered environment as well as racing with static obstacles. The results are additionally validated in a real-time vehicle simulation.


%%%%%%%%%%%%%%%%%%%%%%%%%%%%%%%%%%%%%%%%%%%%%%%%%%%%%%%%%%%%%%%%%%%%%%%%%%%%%%%%

%%%%%%%%%%%%%%%%%%%%%%%%%%%%%%%%%%%%%%%%%%%%%%%%%%%%%%%%%%%%%%%%%%%%%%%%%%%%%%%%

%%%%%%%%%%%%%%%%%%%%%%%%%%%%%%%%%%%%%%%%%%%%%%%%%%%%%%%%%%%%%%%%%%%%%%%%%%%%%%%%
%\section*{APPENDIX}


\section*{ACKNOWLEDGMENT}

This work was supported in part by the NSF Safe Learning Enabled Systems Program, and the Strategic Hub for International Research Collaboration project of Seoul National University. We thank Javier Borquez for sharing the source code of DualGuard-MPPI and valuable insights for the development of this work.

\bibliographystyle{IEEEtran}
\bibliography{references}

@inproceedings{mizuta2026cfmmppi,
  title={Unified Generation-Refinement Planning: Bridging Guided Flow Matching and Sampling-Based MPC for Social Navigation},
  author={Mizuta, Kazuki and Leung, Karen},
  booktitle={IEEE International Conference on Robotics and Automation (ICRA)},
  year={2026}
}

@inproceedings{carron2023chronos,
  title={Chronos and CRS: Design of a miniature car-like robot and a software framework for single and multi-agent robotics and control},
  author={Carron, Andrea and Bodmer, Sabrina and Vogel, Lukas and Zurbr{\"u}gg, Ren{\'e} and Helm, David and Rickenbach, Rahel and Muntwiler, Simon and Sieber, Jerome and Zeilinger, Melanie N},
  booktitle={IEEE International Conference on Robotics and Automation (ICRA)},
  year={2023},
  organization={IEEE}
}

@article{wabersich2021predictive,
  title={A predictive safety filter for learning-based control of constrained nonlinear dynamical systems},
  author={Wabersich, Kim Peter and Zeilinger, Melanie N},
  journal={Automatica},
  volume={129},
  pages={109597},
  year={2021},
  publisher={Elsevier}
}

@article{homburger2025optimality,
  title={Optimality and suboptimality of MPPI control in stochastic and deterministic settings},
  author={Homburger, Hannes and Messerer, Florian and Diehl, Moritz and Reuter, Johannes},
  journal={IEEE Control Systems Letters},
  year={2025},
  publisher={IEEE}
}

@article{williams2017model,
  title={Model predictive path integral control: From theory to parallel computation},
  author={Williams, Grady and Aldrich, Andrew and Theodorou, Evangelos A},
  journal={Journal of Guidance, Control, and Dynamics},
  volume={40},
  number={2},
  pages={344--357},
  year={2017},
  publisher={American Institute of Aeronautics and Astronautics}
}

@inproceedings{williams2017information,
  title={Information theoretic MPC for model-based reinforcement learning},
  author={Williams, Grady and Wagener, Nolan and Goldfain, Brian and Drews, Paul and Rehg, James M and Boots, Byron and Theodorou, Evangelos A},
  booktitle={2017 IEEE international conference on robotics and automation (ICRA)},
  pages={1714--1721},
  year={2017},
  organization={IEEE}
}

@article{yin2023shield,
  title={Shield model predictive path integral: A computationally efficient robust mpc method using control barrier functions},
  author={Yin, Ji and Dawson, Charles and Fan, Chuchu and Tsiotras, Panagiotis},
  journal={IEEE Robotics and Automation Letters},
  volume={8},
  number={11},
  pages={7106--7113},
  year={2023},
  publisher={IEEE}
}

@article{borquez2025dualguard,
  title={Dualguard mppi: Safe and performant optimal control by combining sampling-based mpc and hamilton-jacobi reachability},
  author={Borquez, Javier and Raus, Luke and Ciftci, Yusuf Umut and Bansal, Somil},
  journal={IEEE Robotics and Automation Letters},
  year={2025},
  publisher={IEEE}
}

@book{rawlings2017,
	address = {Madison, Wisconsin},
	edition = {2nd edition},
	title = {{Model predictive control: theory, computation, and design}},
	isbn = {978-0-9759377-3-0},
	shorttitle = {Model predictive control},
	language = {en},
	publisher = {Nob Hill Publishing},
	author = {Rawlings, J. B. and Mayne, D. Q. and Diehl, M.},
	year = {2017}
	}

@article{verschueren2022acados,
  title={acados—a modular open-source framework for fast embedded optimal control},
  author={Verschueren, Robin and Frison, Gianluca and Kouzoupis, Dimitris and Frey, Jonathan and Duijkeren, Niels van and Zanelli, Andrea and Novoselnik, Branimir and Albin, Thivaharan and Quirynen, Rien and Diehl, Moritz},
  journal={Mathematical Programming Computation},
  volume={14},
  number={1},
  pages={147--183},
  year={2022},
  publisher={Springer}
}

@article{boggs1995sequential,
  title={Sequential quadratic programming},
  author={Boggs, Paul T and Tolle, Jon W},
  journal={Acta numerica},
  volume={4},
  pages={1--51},
  year={1995},
  publisher={Cambridge University Press}
}

@article{fazlyab2026model,
  title={Model Predictive Path Integral Control as Preconditioned Gradient Descent},
  author={Fazlyab, Mahyar and Sharifi, Sina and Wang, Jiarui},
  journal={arXiv preprint arXiv:2603.24489},
  year={2026}
}

@inproceedings{bansal2017hamilton,
  title={Hamilton-jacobi reachability: A brief overview and recent advances},
  author={Bansal, Somil and Chen, Mo and Herbert, Sylvia and Tomlin, Claire J},
  booktitle={2017 IEEE 56th annual conference on decision and control (CDC)},
  pages={2242--2253},
  year={2017},
  organization={IEEE}
}

@article{liniger2015optimization,
  title={Optimization-based autonomous racing of 1: 43 scale RC cars},
  author={Liniger, Alexander and Domahidi, Alexander and Morari, Manfred},
  journal={Optimal Control Applications and Methods},
  volume={36},
  number={5},
  pages={628--647},
  year={2015},
  publisher={Wiley Online Library}
}

@inproceedings{fisac2019bridging,
  title={Bridging hamilton-jacobi safety analysis and reinforcement learning},
  author={Fisac, Jaime F and Lugovoy, Neil F and Rubies-Royo, Vicen{\c{c}} and Ghosh, Shromona and Tomlin, Claire J},
  booktitle={IEEE International Conference on Robotics and Automation (ICRA)},
  pages={8550--8556},
  year={2019},
  organization={IEEE}
}

@inproceedings{yin2025safe,
  title={Safe beyond the horizon: Efficient sampling-based mpc with neural control barrier functions},
  author={Yin, Ji and So, Oswin and Yu, Eric Yang and Fan, Chuchu and Tsiotras, Panagiotis},
  booktitle={Robotics: Science and Systems},
  year={2025}
}

@inproceedings{tao2022control,
  title={Control barrier function augmentation in sampling-based control algorithm for sample efficiency},
  author={Tao, Chuyuan and Kim, Hunmin and Yoon, Hyungjin and Hovakimyan, Naira and Voulgaris, Petros},
  booktitle={American Control Conference (ACC)},
  pages={3488--3493},
  year={2022},
  organization={IEEE}
}

@inproceedings{yin2024chance,
  title={Chance-constrained information-theoretic stochastic model predictive control with safety shielding},
  author={Yin, Ji and Tsiotras, Panagiotis and Berntorp, Karl},
  booktitle={2024 IEEE 63rd Conference on Decision and Control (CDC)},
  pages={653--658},
  year={2024},
  organization={IEEE}
}

@inproceedings{di2018dynamic,
  title={Dynamic locomotion in the mit cheetah 3 through convex model-predictive control},
  author={Di Carlo, Jared and Wensing, Patrick M and Katz, Benjamin and Bledt, Gerardo and Kim, Sangbae},
  booktitle={IEEE/RSJ international conference on intelligent robots and systems (IROS)},
  year={2018},
  organization={IEEE}
}

@article{pankert2020perceptive,
  title={Perceptive model predictive control for continuous mobile manipulation},
  author={Pankert, Johannes and Hutter, Marco},
  journal={IEEE Robotics and Automation Letters},
  volume={5},
  number={4},
  year={2020},
  publisher={IEEE}
}

@article{williams2018information,
  title={Information-theoretic model predictive control: Theory and applications to autonomous driving},
  author={Williams, Grady and Drews, Paul and Goldfain, Brian and Rehg, James M and Theodorou, Evangelos A},
  journal={IEEE Transactions on Robotics},
  volume={34},
  number={6},
  pages={1603--1622},
  year={2018},
  publisher={IEEE}
}

@article{rawlik2012stochastic,
  title={On stochastic optimal control and reinforcement learning by approximate inference},
  author={Rawlik, Konrad and Toussaint, Marc and Vijayakumar, Sethu},
  journal={Robotics: Science and Systems},
  year={2012}
}

@misc{wang2026handbook,
  title = {Handbook on Safety Certificates: Techniques from Hamilton--Jacobi Reachability Analysis and Control Barrier Functions},
  author = {Hao Wang and Max Cohen and Zeyuan Feng and David Oh and Oswin So and Songyuan Zhang and Javier Borquez and Chuchu Fan and Jaime Fisac and Haimin Hu and Somil Bansal and Jason J. Choi},
  year = {2026},
  note = {Preprint}
}

@article{CohenARC24,
  author          = {M. H. Cohen and T. G. Molnar and A. D. Ames},
  journal         = {Annual Reviews in Control},
  number          = {},
  title           = {Safety-Critical Control for Autonomous Systems: Control Barrier Functions via Reduced Order Models},
  volume          = {57},
  year            = {2024},
  pages = {100947},
}

@inproceedings{Ames2019,
author  = {Ames, Aaron D. and Coogan, Samuel and Egerstedt, Magnus and Notomista, Gennaro and Sreenath, Koushil and Tabuada, Paulo},
booktitle = procECC,
title   = {Control Barrier Functions: Theory and Applications},
year    = {2019},
pages   = {3420-3431},
}

@article{feng2025bridging,
  title={Bridging model predictive control and deep learning for scalable reachability analysis},
  author={Feng, Zeyuan and Qiu, Le and Bansal, Somil},
  journal={Robotics: Science and Systems},
  year={2025}
}

@inproceedings{lin2024verification,
  title={Verification of neural reachable tubes via scenario optimization and conformal prediction},
  author={Lin, Albert and Bansal, Somil},
  booktitle={6th Annual Learning for Dynamics \& Control Conference},
  pages={719--731},
  year={2024},
  organization={PMLR}
}

@INPROCEEDINGS{choi2025datadriven,
  author={Choi, Jason J. and Strong, Christopher A. and Sreenath, Koushil and Cho, Namhoon and Tomlin, Claire J.},
  booktitle=procCDC, 
  title={Data-Driven Hamiltonian for Direct Construction of Safe Set from Trajectory Data}, 
  year={2025}
 }

@article{
    ros2,
    author = {Steven Macenski  and Tully Foote  and Brian Gerkey  and Chris Lalancette  and William Woodall },
    title = {Robot Operating System 2: Design, architecture, and uses in the wild},
    journal = {Science Robotics},
    volume = {7},
    number = {66},
    pages = {eabm6074},
    year = {2022},
    doi = {10.1126/scirobotics.abm6074},
    URL = {https://www.science.org/doi/abs/10.1126/scirobotics.abm6074}
}

\appendix
We provide an ablation study of the different interconnections of MPPI and MPC discussed in Section~\ref{sec:method} for the simulations in Section~\ref{sec:simulations} in Tables~\ref{tab:cluttered_ablation},~\ref{tab:kin_ablation}~and~\ref{tab:dyn_ablation}.
\begin{table}[h]
\vspace{-0.5em}
\centering
\caption{Cluttered Environment Ablation Study}
\vspace{-0.5em}
\begin{tabular}{cccc}
 & Avg. Cost & Succ. [\%] & Coll. [\%] \\ \hline
MPC (SQP)      & $1.039 \pm 0.089$ & 26 & 71 \\ \hline
Dualguard MPPI \cite{borquez2025dualguard}     & $1.554 \pm 0.124$ & 91 & \textbf{0}  \\ \hline
\makecell{\textbf{\ourmethod} \\ (ws. + ref. + nom.)}      & $1.000 \pm 0.102$ & \textbf{96} & \textbf{0}  \\ \hline
\makecell{\ourmethod \\ (ref.)} & $0.978 \pm 0.120$ & 43 & 55 \\ \hline
\makecell{\ourmethod \\ (ref. + nom.)}    & $1.150 \pm 0.105$ & 63 & 34 \\ \hline
\makecell{\ourmethod \\ (ws.)} & $\textbf{0.946} \pm 0.104$ & 82 & 15 \\ \hline
\makecell{\ourmethod \\ (ws. + nom.)}     & $0.955 \pm 0.100$ & 85 & 9  \\ \hline
\makecell{\ourmethod \\ (ws. + ref.)}     & $1.043 \pm 0.115$ & 51 & 43 \\
\end{tabular}
\label{tab:cluttered_ablation}
\vspace{-1em}
\end{table}
\begin{table}[h]
\centering
\vspace{-0.8em}
\caption{Kinematic Model Ablation Study}
\vspace{-0.5em}
\begin{tabular}{cccc}
 & Avg. Lap Time [s] & Succ. [\%] & Coll. [\%] \\ \hline
MPC (SQP)      & $4.274 \pm 0.020$ & 32  & 68 \\ \hline
Dualguard MPPI \cite{borquez2025dualguard}      & $5.831 \pm 0.210$ & 96  & 4  \\ \hline
  \makecell{\ourmethod \\ (ws. + ref. + nom.)}    & $5.298 \pm 0.146$ & 100 & 0  \\ \hline
\makecell{\ourmethod \\ (ref.)} & $5.362 \pm 0.142$ & 100 & 0  \\ \hline
\makecell{\ourmethod \\ (ref. + nom.)}    & $5.282 \pm 0.149$ & 98  & 2  \\ \hline
\makecell{\ourmethod \\ (ws.)} & $\textbf{3.867} \pm 0.010$ & 100 & 0  \\ \hline
\makecell{\textbf{\ourmethod} \\ (ws. + nom.)}     & $3.884 \pm 0.012$ & 100 & 0  \\ \hline
\makecell{\ourmethod \\ (ws. + ref.)}     & $5.347 \pm 0.139$ & 100 & 0  \\ 
\end{tabular}
\label{tab:kin_ablation}
\vspace{-1em}
\end{table}
\begin{table}[h]
\vspace{-0.8em}
\centering
\caption{Dynamic Model Ablation Study}
\vspace{-0.5em}
\begin{tabular}{cccc}
 & Avg. Lap Time [s] & Succ. [\%] & Coll. [\%] \\ \hline
MPC (SQP)      & $7.629 \pm 0.124$ & 42 & 58 \\ \hline
MPPI      & -                 & 0  & 38 \\  \hline
\makecell{\textbf{\ourmethod} \\ (ws. + ref. + nom.)}      & $7.546 \pm 0.198$ & \textbf{98} & \textbf{0}  \\ \hline
\makecell{\ourmethod \\ (ref.)} & $7.245 \pm 0.185$ & 72 & 26 \\ \hline
\makecell{\ourmethod \\ (ref. + nom.)}    & $7.204 \pm 0.143$ & 92 & 8  \\ \hline
\makecell{\ourmethod \\ (ws.)} & $\textbf{7.177} \pm 0.106$ & 46 & 46 \\ \hline
\makecell{\ourmethod \\ (ws. + nom.)}     & $7.572 \pm 0.167$ & 58 & 42 \\ \hline
\makecell{\ourmethod \\ (ws. + ref.)}     & $7.511 \pm 0.204$ & 82 & 8  \\ \hline
\makecell{\ourmethod \\ (kin. model MPPI)} & $8.668 \pm 0.172$ & 90 & 10 \\
\end{tabular}
\label{tab:dyn_ablation}
\vspace{-1em}
\end{table}

\end{document}